\documentclass[sigconf, nonacm]{aamas} 

\usepackage{balance} 

\usepackage{graphicx}
\usepackage{sidecap}
\usepackage{tikz}
\usetikzlibrary{matrix,decorations.pathreplacing}

\theoremstyle{acmplain}

\newcommand{\ceil}[1]{\left\lceil #1 \right\rceil}
\newcommand{\floor}[1]{\left\lfloor #1 \right\rfloor}

\setcopyright{ifaamas}
\acmConference[AAMAS '22]{Proc.\@ of the 21st International Conference
on Autonomous Agents and Multiagent Systems (AAMAS 2022)}{May 9--13, 2022}
{Online}{P.~Faliszewski, V.~Mascardi, C.~Pelachaud,
M.E.~Taylor (eds.)}
\copyrightyear{2022}
\acmYear{2022}
\acmDOI{}
\acmPrice{}
\acmISBN{}

\acmSubmissionID{471}

\title{The Price of Majority Support}

\author{Robin Fritsch}
\affiliation{
  \institution{ETH Zürich}
  \country{Switzerland}}
\email{rfritsch@ethz.ch}

\author{Roger Wattenhofer}
\affiliation{
  \institution{ETH Zürich}
  \country{Switzerland}}
\email{wattenhofer@ethz.ch}

\begin{abstract}
We consider the problem of finding a compromise between the opinions of a group of individuals on a number of mutually independent, binary topics.
In this paper, we quantify the loss in representativeness that results from requiring the outcome to have majority support, in other words, the ``price of majority support''.
Each individual is assumed to support an outcome if they agree with the outcome on at least as many topics as they disagree on.
Our results can also be seen as quantifying Anscombes paradox which states that topic-wise majority outcome may not be supported by a majority.
To measure the representativeness of an outcome, we consider two metrics.
First, we look for an outcome that agrees with a majority on as many topics as possible.
We prove that the maximum number such that there is guaranteed to exist an outcome that agrees with a majority on this number of topics and has majority support, equals $\ceil{(t+1)/2}$ where $t$ is the total number of topics.
Second, we count the number of times a voter opinion on a topic matches the outcome on that topic. The goal is to find the outcome with majority support with the largest number of matches.
We consider the ratio between this number and the number of matches of the overall best outcome which may not have majority support.
We try to find the maximum ratio such that an outcome with majority support and this ratio of matches compared to the overall best is guaranteed to exist.
For 3 topics, we show this ratio to be $5/6\approx 0.83$.
In general, we prove an upper bound that comes arbitrarily close to $2\sqrt{6}-4\approx 0.90$ as $t$ tends to infinity.
Furthermore, we numerically compute a better upper and a non-matching lower bound in the relevant range for $t$.
\end{abstract}


\newcommand{\BibTeX}{\rm B\kern-.05em{\sc i\kern-.025em b}\kern-.08em\TeX}

\begin{document}


\pagestyle{fancy}


\settopmatter{printfolios=true}
\maketitle 


\section{Introduction}
How can a group of individuals find a compromise on a given set of disputed topics? This central problem in social choice theory comes up in a large number of applications. For instance, which motions should a political party support in an upcoming election? Or which candidates should be elected as committee members? 

More precisely, we consider a situation in which decisions on a number of mutually independent, binary topics are to be made, and each voter has an opinion on each topic.
For example, consider the situation shown in Figure \ref{fig:intro_example}.

\begin{figure}[h]
\centering
\begin{tikzpicture}[decoration={brace,amplitude=5pt}]
    \matrix (m) [matrix of nodes, column sep = 4.5pt, row sep = -2pt] {
    Y & Y & Y & Y & Y & Y & Y\\
    Y & Y & Y & Y & Y & Y & Y\\
    Y & Y & Y & Y & Y & Y & Y\\
    Y & Y & Y & Y & Y & Y & Y\\
    Y & Y & Y & N & N & N & N\\
    N & N & N & Y & Y & N & N\\
    N & N & N & Y & Y & N & N\\
    N & N & N & N & N & Y & Y\\
    N & N & N & N & N & Y & Y\\
    };
\end{tikzpicture}
\caption{The opinions of 9 voters (rows) on 7 binary topics (columns). In this matrix, the $j$-th entry in the $i$-th row indicates whether voter $i$ has opinion Y (``Yes'') or N (``No'') on topic $j$.}
\label{fig:intro_example}
\end{figure}

In this example, a majority of at least 5 voters supports Y on every topic.
Therefore, the natural decision for the group would be to choose Y on every topic.
However, the bottom 5 would disagree with this Y-only decision. These bottom 5 voters would in fact rather support the complete opposite -- an N-only decision. And the bottom 5 voters are in the majority!

Such lack of majority support clearly poses a problem. Usually majority support is a natural must-have criterion.
If the outcome of difficult negotiations is not supported by a majority, it will not be ratified.

This motivates to study the restricted problem in which only outcomes with a majority support are acceptable.
What is the best outcome under this restriction and how well can it represent the voter opinions?
Finding a proposal with majority support is simple: if some (any!) proposal is not supported by a majority, its opposite will be. Finding a good proposal with majority support however turns out to be difficult.
In this paper, we will study guarantees on how well a proposal with majority support can represent the voter opinions according to two metrics.

First, we look for a proposal with majority support that agrees with a majority on as many topics as possible.
In our example, the proposal YYYYYYN (among others) has majority support, and also agrees with a majority on 6 of the 7 topics.
In general, how low can this number be for the best compromise supported by a majority?
We show the tight result that this number can be as low as $\ceil{(t+1)/2}$ when $t$ is the number of topics.

However, counting the topics on which the result agrees with a majority does not take into account that some majorities can be larger than others.
To better measure how well a decision represents the voter opinions, we consider a second metric. This second quality metric counts the number of times a voter opinion on a topic matches the proposal on that topic.
In the example above, there are a total of 39 matching opinions between the voters and the Y-only proposal, the most among all possible proposals. We call such a proposal the topic-wise majority. However, as discussed before the topic-wise majority proposal Y-only does not have majority support.
The goal is to find a proposal with majority support with the largest number of matching opinions with the voters. In the example, this would be (among others) NNYYYYY with 37 matches.

The ratio between the best proposal supported by a majority and the topic-wise majority proposal can be seen as the ``price of majority support".
This paper aims to quantify this price of majority support.
In the example, the best proposal with majority support achieves a ratio of $37/39\approx 0.95$ of the maximum possible matches.
We try to find the maximum ratio such that a compromise with this ratio and with majority support is guaranteed to exist.
For 3 topics, we show this ratio to be $5/6\approx 0.83$.
In general, we prove an upper bound that comes arbitrarily close to $2\sqrt{6}-4\approx 0.90$ as $t$ tends to infinity.
Furthermore, we can numerically compute a stronger upper bound which equals approximately $0.79$ for large $t$ (i.e.\ above 1000).
We also compute a non-matching lower bound that is within a range of $0.11$ of the upper bound.

Our results also apply to the question of how well a population of voters can be represented by two parties in a representative democracy.
We assume the election campaign focuses on a number of binary topics and voters vote for the party with which they share the most opinions.
If the two parties have the same opinion on a topic, this topic will be irrelevant for the voter's choice of a party. Without loss of generality we can therefore only consider topics on which the two parties have opposing views. 

For the voter opinions shown in Figure \ref{fig:intro_example}, a party supporting the majority opinion on every topic (YYYYYYY), which is arguably the optimal outcome of the election and the result in direct democracy, would lose the election.
Instead, the opposing party supporting the minority opinion on each of these topics (NNNNNNN) would win.
How well a representative democracy can work in the best case is then equivalent to quantifying the price of majority support.
More precisely, it is equivalent to studying how close a  party winning the election can come to the (optimal) party supporting the majority opinion on each topic.

The results in this paper can be understood as a mathematical treatment of the power and limitations of political ``populism'', i.e.\ the limits of a party that tries to always follow the majority opinions.

\section{Related Work}

If the binary decisions to be made are whether or not to elect a certain candidate, our setting is equivalent to deciding on a committee or assembly by approval voting \cite{brams07}.
For approval voting, each voter indicates their preferences by choosing a set of candidates they approve. In other words, every voter decides for every candidate whether or not to approve them. Based on these preferences a set of winning candidates is chosen.
Approval voting is particularly popular for committee elections, and has been used by many organizations, including also scientific societies \cite{brams05}.

The most natural voting rule in approval voting is the so-called minisum rule which selects a committee (possibly of a fixed size) that minimizes the sum of hamming distances to the voters.
This is equivalent to our formulation of maximizing the number of matching opinions with voters.

Besides minisum, a variety of different voting rules for approval voting have been studied \cite{kilgour10}.
One example is the 
minimax solution that minimizes the maximum hamming distance to the voters instead of the sum \cite{brams07minimax}.

The situation in which more than half the voters disagree with the topic-wise majority on more than half the topics is known as the Anscombe paradox, introduced by Gertrude Elizabeth Margaret Anscombe \cite{anscombe76}.
Several very similar voting paradoxes have also been studied \cite{nurmi99, saari01, laffond10}, e.g.\ the well-known Condorcet paradox \cite{condorcet85, gehrlein83}, the Ostrogorski paradox \cite{ostrogorski02, rae76}, and the paradox of multiple elections \cite{brams98, scarsini98}.
The occurrence of such paradoxes has also been studied empirically \cite{kaminski15, kurrild18}.

Previous research has studied sufficient conditions on the voter matrix such that guarantee the Anscombe paradox will not occur.
The ``Rule of Three-Fourth'' \cite{wagner83} states the Anscombe paradox will not occur if the average majority on each topic is at least $3/4$.
This rule was later generalized to the ``Rule of $(1-\alpha/2)$'' which prevents the $\alpha$-Anscombe paradox \cite{wagner84}.
The latter describes the situation when a proportion $\alpha$ of voters disagree with topic-wise majority proposal.
Two further conditions that prevent this kind of paradox are the single-switch property \cite{laffond06} and an unanimity condition \cite{laffond13}.
The latter is a limit on how much the opinions of any two voters may differ.

All these results have in common that they describe sufficient conditions on the voter matrix which prevent the paradox.
These are however of limited use in practice since it is not clear if a certain group of voters satisfies the conditions.
In this paper, we take a different approach and examine which outcomes that avoid the paradox are closest to minisum.
In other words, we present the first \textit{quantitative Anscombe} results. Our results answer the question of how bad the Anscombe paradox can get, i.e.\ how much a result needs to be changed to avoid the paradox.
Note that our approach can be applied in practice as a voting rule for approval voting: Instead of selecting the minisum outcome, one could choose the closest-to-minisum compromise that has majority support.
Our results then provide guarantees on how close the outcome is to the topic-wise majority result, i.e. how well it represents the voter opinions.

Note that the voting rule above (with the second representativeness metric we consider) is equivalent to the Slater rule \cite{slater61, nehring14} in judgement aggregation \cite{list09}.
Judgement aggregation however, studies issues that can be logically interconnected while the issues in our model are assumed to be independent.

\section{Formal Model}

We consider a situation in which $n\in\mathbb{N}$ voters need to make decisions on $t\in\mathbb{N}$ binary topics.
The two alternatives for each topic are denoted Y (``Yes'') and N (``No'').
We assume each voter has an opinion on each topic meaning each voter's opinions can be written as a vector $v\in\{\text{Y},\text{N}\}^t$. 

All voter opinions are summed up in the \emph{voter matrix} $V\in \{Y,N\}^{n\times t}$ in which the $j$-th entry of the $i$-th row indicates the opinion of voter $i$ on topic $j$. (Figure \ref{fig:intro_example} shows an example.)
For each topic, we call the opinion that is held by a majority of voters the \emph{majority opinion}. Without loss of generality, we assume that the majority opinion is Y for each topic, i.e.\ there are at least as many Ys in any column as there are Ns.
Let $\mathcal{V}_t$ be the set of all voter matrices for $t$ topics and an arbitrary number of voters.

Furthermore, every possible \emph{proposal} is represented by a vector $p\in \{\text{Y},\text{N}\}^t$.
We assume that a voter $v$ supports a proposal $p$ if the hamming distance between $v$ and $p$ is no larger than $t/2$. In other words, if the voter agrees with a proposal on at least as many topics as they disagree on, the voter supports this proposal.
We say a proposal is supported by $V$ if at least half of all voters support the proposal.
For a given proposal, we count the number of \emph{matches} between an entry in the voter matrix and the entry in the proposal for that topic.
The maximum possible number of matches is $nt$.


For a proposal $p$, the \emph{opposite proposal} $\bar{p}$ chooses the opposite option on each topic.
We call a proposal whose vector contains exactly $k$ Ys a $k$-proposal for $k=0,\ldots, t$, and similarly a voter whose vector contains exactly $k$ Ys as a $k$-voter.
Finally, let $v_k$ be the number of $k$-voters in a voter matrix $V$.

\section{Most Majority Decisions}

We begin by looking for the proposal that agrees with as many majority opinions as possible among the proposals with majority support.

\begin{definition}
For a voter matrix $V$, let $md_V$ be the maximum number such that there exists a proposal supported by $V$ with $md_V$ majority decisions.
Furthermore, let $md_t = \min_{V\in \mathcal{V}_t} md_V$.
\end{definition}

In other words, $md_t$ is the maximum number, such that for every voter matrix with $t$ topics, a proposal with $md_t$ majority decisions that is supported by a majority of voters exists.
From the Anscombe paradox we know that the proposal consisting of all $t$ majority opinions might not be supported by a majority, i.e.\ $md_t<t$.
In the following we will prove that $md_t=\ceil{(t+1)/2}$.

\begin{lemma}\label{lem:md_t_upper_bound}
It holds that $md_t \leq \ceil{\frac{t+1}{2}}$.
\end{lemma}
\begin{proof}
Since $md_t = \min_{V\in \mathcal{V}_t} md_V$, it suffices to show that for every $t$, the following voter matrix $W$ satisfies $md_W= \ceil{(t+1)/2}$.

\begin{figure}[h]
\centering
\begin{tikzpicture}[decoration={brace,amplitude=5pt}]
    \matrix (m) [matrix of nodes, column sep = 4.5pt, row sep = -2pt] {
    Y & N & $\cdots$ & N\\
    N & Y & $\cdots$ & N\\
    $\vdots$ & & $\ddots$ & $\vdots$\\
    N & N & $\cdots$ & Y\\
    Y & Y & $\cdots$ & Y\\
    $\vdots$ & & $\ddots$ & $\vdots$\\
    Y & Y & $\cdots$ & Y\\
    };
    \draw[decorate,thick] ([yshift=-2pt] m-1-4.north east) -- node[right=5pt] {$t$} ([yshift=2pt] m-4-4.south east);
    \draw[decorate,thick] ([yshift=-2pt] m-5-4.north east) -- node[right=5pt] {$t-1$} ([yshift=2pt] m-7-4.south east);
\end{tikzpicture}
\end{figure}

Any proposal containing at least $\ceil{(t+1)/2}+1$ Ys will disagree on at least $\ceil{(t+1)/2}$ topics with any of the first $t$ voters.
Hence, no such proposal will be supported by a majority. On the other hand, any $\ceil{(t+1)/2}$-proposal is clearly supported by a majority.
\end{proof}


Proving the sharp lower bound of $\ceil{(t+1)/2}$ is more involved.
In order to do so, we introduce the following notation which is also essential for the next section.

\begin{definition}
Let $s_{k,l}$ be the number of $k$-proposals (among all $\binom{t}{k}$ possible $k$-proposals) which are supported by an $l$-voter.
\end{definition}

As an example, consider the case $t=3$.
Here the $1$-voter YNN supports the two $2$-proposals YYN and YNY (but not NYY). So we have $s_{2,1}=2$.

Note that $s_{k,l}$ is well-defined as all $l$-voters support the same number of $k$-proposals.
In the example above, the other two $1$-voters also support two $2$-proposals: NYN supports YYN and NYY while NNY supports NYY and YNY.

To see that in general any two $l$-voters support the same number of $k$-proposals, note that there exists a permutation of topics that transforms one of the $l$-voters into the other. Furthermore, the set of all $k$-proposals is invariant under any permutation of the topics.

In the following, we give an explicit formula for $s_{k,l}$.

\begin{lemma}\label{lem:s_kl}
For any $k\in\{ \ceil{(t+1)/2}, \ldots, t\}$ and $l\in\{0,\ldots, t\}$, we have
\begin{equation*}
    s_{k,l} = \sum_{x=\ceil{\frac{k+l-\floor{t/2}}{2}}}^{k} \binom{l}{x}\binom{t-l}{k-x}.
\end{equation*}
\end{lemma}
\begin{proof}
Consider an arbitrary $l$-voter.
For $x=0,\ldots,k$, consider all $k$-proposal that have exactly $x$ matches with the $l$-voter on its Ys. 
There are $\binom{l}{x} \binom{t-l}{k-x}$ such proposals.
These proposals match the voter on $x$ Ys as well as on $(t-l)-(k-x)$ Ns.
Hence, the $l$-voter supports the $k$-proposal if and only if $x+(t-l)-(k-x)\geq \ceil{t/2}$.
The latter is equivalent to $x\geq \ceil{(k+l-\floor{t/2})/2}$.
\end{proof}

\begin{lemma}\label{lem:skl_inequality}
Let $V$ be a voter matrix with $md_V<w$ for some $w\geq \ceil{(t+1)/2}$. Then
\begin{equation}\label{eq:lp_ineq}
    s_{k,0}v_0 + s_{k,1}v_1+ \ldots + s_{k,t}v_t \leq s_{k,t} \floor{\frac{n-1}{2}}
\end{equation}
for $k=w,\dots,t$.
\end{lemma}
\begin{proof}
The term on the left side of the inequality represents the sum over all $k$-proposals of the number of supporter each of these proposals has among the voters in $V$.
Since there are $\binom{t}{k}$ $k$-proposals and none of them is supported by $V$, i.e. each of them is supported by at most $\floor{(n-1)/2}$ voters, this number is no larger than $\binom{t}{k} \floor{(n-1)/2}$.
Finally, it remains to verify that $\binom{t}{k}=s_{k,t}$ for $k\geq \ceil{(t+1)/2}$.
This is true since a voter opinion Y on all $t$ topics supports all proposals with at least $\ceil{(t+1)/2}$ Ys.
\end{proof}

With this, we are now ready to prove the main theorem of this section which gives a tight lower bound on $md_t$.

\begin{theorem}\label{thm:mdt_lower}
It holds that $md_t\geq \ceil{\frac{t+1}{2}}$.
\end{theorem}
\begin{proof}
We first prove the statement for odd $t$ where it is equivalent to $md_t\geq \ceil{t/2}$.
To do so, we need to show that $md_V \geq \ceil{t/2}$ for all voter matrices $V$.
Assume there exists a voter matrix $V$ such that $md_V < \ceil{t/2}$.
Then inequality \eqref{eq:lp_ineq} holds for $k=\ceil{t/2},\ldots,t$.
For each of these $k$, we multiply the inequality with $2k-t$ and finally sum them all up.
This results in the inequality
\begin{equation}\label{eq:c_l}
    c_0v_0 + c_1v_1 +\ldots + c_tv_t \leq c_t \floor{\frac{n-1}{2}}
\end{equation}
with coefficients
\begin{equation*}
    c_l = \sum_{k = \ceil{t/2}}^{t} (2k-t) s_{k,l}
\end{equation*}
for $l=0,\ldots,t$.

Combinatorial calculations show that $c_l = l \binom{t-1}{\floor{t/2}}$.
For details, see Lemma \ref{lem:bin_sum_app} in the Appendix.
Inserting this into \eqref{eq:c_l} yields $\sum_{l=0}^t l v_l\leq t\floor{(n-1)/2}$. Note that the expression on the left side of the latter inequality represents the number of Ys in the matrix $V$. Since we assumed there is a majority of Ys in every column of the matrix, this number must be at least $nt/2$ which is a contradiction.

Finally, let $t$ be even. If we only consider the first $t-1$ topics, there exists a proposal with majority support with $r_{t-1}$ Ys.
Each voter that support this proposal will still support it when considering all $t$ topics no matter what the proposal suggests for topic $t$.
Using the statement for odd $t$, we conclude $md_t\geq 1+md_{t-1} \geq \ceil{(t+1)/2}$ for even $t$.
\end{proof}

\section{Best Representation}

In the previous section, we saw that only the existence of a proposal that decides slightly more than half of all topics in favor of a majority and has majority support can be guaranteed.
However, siding with a minority on a large number of topics is not actually that bad if the opposing majorities are very slim.
Note that this is the case for the voter matrices that provide the upper bound on $md_t$ in Lemma \ref{lem:md_t_upper_bound}.
On the other hand, siding with the minority on a few topics with very large majorities could be considered worse.
To quantify these differences of minority decisions, we consider the following representativeness metric.

For every $i=1,\ldots,t$, let $m_i\in[0.5,1]$ be the size of the majority on topic $i$, i.e.\ the fraction of Ys in the $i$-th column of the voter matrix.
Let $m_i'$ be the fraction of voters agreeing with a given proposal on topic $i$.
We have either $m_i'=m_i$ if the proposal supports the majority opinion on topic $i$, or $m_i'=1-m_i$ otherwise.
We call
\begin{equation*}
    R_p = \frac{m_1'+\ldots +m_t'}{t}
\end{equation*}
the \emph{absolute representativeness} of a proposal $p$.
Note that $R_p$ can also be interpreted as the fraction of all $nt$ voter opinions that match the proposal on that topic.
The total number of such matches can also be written as $nt -  \sum_{i=1}^n h(v_i,p)$ where $h(v_i,p)$ denotes the hamming distance between voter $i$ and proposal $p$. Thus,
\begin{equation*}
    R_p = 1 - \frac{1}{nt}\sum_{i=1}^n h(v_i,p).
\end{equation*}
So maximising $R_p$ is equivalent to minimizing the sum of hamming distances. The solution of this minimization is the topic-wise majority proposal and is sometimes referred to as minisum in approval voting \cite{kilgour10}.

The highest possible absolute representativeness of a proposal for a given voter matrix $V$ is $m_V:=(m_1+\ldots +m_t)/t$, i.e.\ the average majority on a topic, which is achieved by the topic-wise majority proposal.
Since this value differs from voter matrix to voter matrix, we consider the \emph{relative representation of $p$}:
\begin{equation*}
    r_p = \frac{R_p}{m_V} = \frac{m_1'+\ldots +m_t'}{m_1+\ldots +m_t}
\end{equation*}

Since the denominator of the fraction is constant for a given voter matrix, maximizing $r_p$ among all proposals still means maximizing the number of matching opinions between the proposal and the voters.
Finally, we denote $R_V=max_p R_p$ and $r_V=max_p r_p$ where in both cases the maximum is taken over all proposals $p$ supported by a majority.
In other words, $R_V$ and $r_V$ are the highest absolute and relative representativeness, respectively, of a proposal with majority support for a given $V$.

\subsection{Analytical Bounds}

Note that $R_p\geq 1/4$ holds for every proposal that is supported by a majority. This is clear, since any such proposal agrees with at least half of all voters on at least half of all topics, i.e. the number of matches is at least $nt/4$.

\sloppy For voter matrices with $m_V\leq 3/4$, this implies $r_p\geq (1/4) / (3/4) = 1/3$ and consequently $r_V\geq 1/3$. On the other hand, we have $r_V=1$ according to the Rule of Three-Fourth if $m_V\geq 3/4$, since the Anscombe paradox does not occur in that case \cite{wagner83}. Together this implies the following lemma.

\begin{lemma}\label{lem:r_V}
For any voter matrix $V$, $r_V\geq 1/3$.
\end{lemma}

In order to establish a better lower bound on $r_V$, we consider a proposal with representativeness of about 0.5 and choose either this proposal or its opposite, depending on which is supported by a majority.

\begin{lemma}\label{lem:1/2prop}
For any voter matrix $V$, we can find a proposal $p$ with $R_p\geq \frac{1}{2}-\frac{1}{t}$.
In particular, this implies $r_V \geq \frac{2}{3} - \frac{4}{3t}$.
\end{lemma}

\begin{proof}
Let $m_1,\ldots,m_t$ be the sizes of the topic-wise majorities.
We choose the maximum $k\in [0, t]$ such that $m_1+\ldots+m_k + (1-m_{k+1}) +\ldots +(1-m_t) \leq t/2$.
Note that this term is non-decreasing as a function of $k$ since $m_i\geq 1-m_i$ for $i=1,\ldots,t$.

Consider the proposal $p=\text{Y}\ldots \text{Y}\text{N}\ldots \text{N}$ consisting of $k$ Ys followed by $n-k$ Ns.
By our choice of $k$, $R_p\leq 1/2$ and therefore $R_{\bar{p}} =1-R_p \geq 1/2$ for the opposite proposal $\bar{p}$.
Furthermore, by the maximality of $k$,
\begin{align*}
    R_p &= \frac{1}{t}\left(m_1+\ldots+m_k + (1-m_{k+1}) +\ldots +(1-m_t) \right)\\
    &= \frac{1}{t}(m_1+\ldots+m_{k+1} + (1-m_{k+2}) +\ldots  +(1-m_t) \\
    &\qquad\qquad\qquad\qquad\qquad\qquad\qquad- m_{k+1} + (1-m_{k+1}))\\
    &> \frac{1}{t}\left( \frac{t}{2} + 1 - 2m_{k+1}\right) \geq \frac{1}{2}-\frac{1}{t}.
\end{align*}
Since either $p$ or $\bar{p}$ is supported by $V$, this implies $R_V\geq 1/2-1/t$.

By the Rule of Three-Fourth, we know that $r_V=1$ if $m_V\geq 3/4$ since the Anscombe paradox does not occur is that case.
Otherwise, if $m_V\leq 3/4$, we have
\begin{equation*}
    r_V \geq \frac{\frac{1}{2}-\frac{1}{t}}{\frac{3}{4}} = \frac{2}{3} - \frac{4}{3t}.
\end{equation*}
\end{proof}

\begin{definition}
Let $r_t = \min_{V\in \mathcal{V}_t} r_V$.
\end{definition}

Intuitively, $r_t$ can be understood as follows.
For any voter matrix with $t$ topics, there exists a proposal with representativeness $r_t$ that is supported by a majority of voters.
For instance, the fact $r_t=0.8$ can be interpreted as the guarantee of the existence of a proposal which is $80\%$ as representative as the optimal proposal and has the support of a majority.
Note that $r_t\in[0,1]$ since $r_V\in[0,1]$. Moreover, the following corollary is a consequence of Lemma \ref{lem:r_V}.

\begin{corollary}
It holds $r_t\geq 1/3$.
\end{corollary}

In the following, our goal is to calculate $r_t$.
It is clear that $r_1=r_2=1$ since the topic-wise majority proposal is always supported by a majority for those $t$ (because $md_1=1$ and $md_2=2$).

\begin{theorem}\label{thm:r_3}
For three topics, we have $r_3 = 5/6$.
\end{theorem}
\begin{proof}
Consider the following matrix $W$.
\begin{figure}[h]
\centering
\begin{tikzpicture}[decoration={brace,amplitude=5pt}]
    \matrix (m) [matrix of nodes, column sep = 4.5pt, row sep = -2pt] {
    Y & N & N\\
    N & Y & N\\
    N & N & Y\\
    Y & Y & Y\\
    };
    \draw[decorate,thick] ([yshift=-2pt] m-1-3.north east) -- node[right=5pt] {$3l$} ([yshift=2pt] m-3-3.south east);
    \draw[decorate,thick] ([yshift=-2pt] m-4-3.north east) -- node[right=5pt] {$3l-1$} ([yshift=2pt] m-4-3.south east);
\end{tikzpicture}
\end{figure}

The proposal YYY is not supported by the first $3l$ voters and thereby does not have majority support.
Since any 2-proposals is supported by a majority,
\begin{equation*}
    r_W = \frac{(4l-1)+(4l-1)+2l}{(4l-1)+(4l-1)+(4l-1)} \to \frac{5}{6}\quad\text{as}\quad l \to \infty.
\end{equation*}
This means $r_3\leq 5/6$.
It remains to prove $r_V\geq 5/6$ for all voter matrices $V$ with 3 topics.
Let $m_1\geq m_2 \geq m_3$ be the majorities on the three topics.

If $m_3\geq 2/3$, we prove by contradiction that the proposal YYY is supported which implies $r_V=1$.
Assume YYY is not supported by $V$, i.e.\ at least $\ceil{n/2}$ voters do not support it. The vector of each of these voters contains at most one Y.
Therefore the fraction of Ys in the voter matrix is at most
\begin{equation*}
    \frac{3\floor{\frac{n}{2}} + \ceil{\frac{n}{2}}}{3n} < \frac{2}{3}.
\end{equation*}
On the other hand, the fraction of Ys in the matrix can also be written as $(m_1+m_2+m_3)/3$. Hence, the assumption $m_3\geq 2/3$ implies that this fraction is at least $2/3$ which is a contradiction.

Otherwise, if $m_3\leq 2/3$, consider the sum of supporters of the proposals YYY and YYN.
A voter with the preferences YN or NY on the first two topics will support exactly one of these two proposals depending on their opinion on the third topic.
Hence, such a voters will contribute exactly 1 to the sum.
Similarly, voters with the preferences NN or YY will contribute exactly 0 and 2, respectively.
In particular, the sum equals the number of Ys in the preferences of the voters on the first two topics.
Since there are at least $n/2+n/2=n$ such Ys, the total support for YYY and YYN is at least $n$ and one of them must be supported by the majority. This implies
\begin{equation*}
    r_V \geq \frac{m_1 + m_2 + (1-m_3)}{m_1+m_2+m_3}\geq \frac{\frac{2}{3}+\frac{2}{3}+\frac{1}{3}}{\frac{2}{3}+\frac{2}{3}+\frac{2}{3}}=\frac{5}{6}.
\end{equation*}
\end{proof}

For general $t\in\mathbb{N}$, we use a generalization of the voter matrix with 3 topics used in the previous proof to find an upper bound on $r_t$.
\begin{theorem}\label{thm:r_t_general}
We have $\lim_{t\to\infty}r_t\leq 2\sqrt{6}-4 \approx 0.8989$.
\end{theorem}

\begin{proof}
Consider a voter matrix with $n=2k+1$ voters where each column contains the same number $M\in[k+1,n]$ of Ys, i.e.\ the majorities on all topics are equally large.
Let the bottom $k$ voters in the matrix each have opinion vectors consisting of $t$ Ys.
The remaining $(M-k)t$ Ys are distributed among the top $k+1$ voters such that all columns contain the same number of Ys, and that the number of Ys of any two of these voters does not differ by more than one.
This means, each of the top $k+1$ voters has at most $x=\ceil{\frac{(M-k)t}{k+1}}$ Ys.
Below such a matrix is shown for $t=5$, $k=4$ and $M=6$.
\begin{figure}[h]
\centering
\begin{tikzpicture}[decoration={brace,amplitude=5pt}]
    \matrix (m) [matrix of nodes, column sep = 4.5pt, row sep = -2pt] {
    Y & N & Y & N & N\\
    Y & N & N & Y & N\\
    N & Y & N & Y & N\\
    N & Y & N & N & Y\\
    N & N & Y & N & Y\\
    Y & Y & Y & Y & Y\\
    Y & Y & Y & Y & Y\\
    Y & Y & Y & Y & Y\\
    Y & Y & Y & Y & Y\\
    };
\end{tikzpicture}
\end{figure}

The larger $M$ is, the more majority opinions a proposal with majority support can agree with, but the worse it is when it disagrees with a majority opinion. In the following, we determine $M$ such that $r_V$ is minimal.

No proposal with $x+\ceil{(t+1)/2}$ or more Ys is supported by any of the top $k+1$ voters leaving it without a majority.
Therefore,
\begin{align*}
    r_V &\leq \frac{\left(x+\ceil{\frac{t+1}{2}}-1\right)M + \left(\floor{\frac{t+1}{2}}-x\right)(n-M)}{tM} \\
    &= 1 - \frac{\left(\floor{\frac{t+1}{2}}-x\right)(2M-n)}{tM}.
\end{align*}
For odd $t$, we have $\floor{(t+1)/2}=t/2+1/2$, and conclude
\begin{align*}
    1-r_V &\geq \frac{\left(\frac{t}{2}+\frac{1}{2}-\left(\frac{(M-k)t}{k+1}+1\right)\right)(2M-n)}{tM}\\
    &= \left(\frac{1}{2}-\frac{M-k}{k+1}\right)\left(2-\frac{2k+1}{M}\right) - \frac{1}{2t}\left(2-\frac{2k+1}{M}\right).
\end{align*}
By choosing $k$ and $M$ sufficiently large for fixed $M/k$, this term comes arbitrarily close to
\begin{equation*}
    \left(\frac{1}{2}-\frac{M-k}{k}\right)\left(2-\frac{2k}{M}\right) - \frac{1}{2t}\left(2-\frac{2k}{M}\right) = 5 - 2\frac{M}{k} - 3\frac{k}{M} - \frac{1}{t}\left(1-\frac{k}{M}\right).
\end{equation*}
Choosing $M/k=\sqrt{3/2}$ maximizes the part independent of $t$, and yields
\begin{equation*}
    r_V\leq 2\sqrt{6}-4+\frac{1}{t}\left(1-\sqrt{\frac{2}{3}}\right).
\end{equation*}
For even $t$, we have $\floor{(t+1)/2}=t/2$, and can use the same steps as above to arrive at
\begin{equation*}
    r_V\leq 2\sqrt{6}-4+\frac{2}{t}\left(1-\sqrt{\frac{2}{3}}\right).
\end{equation*}

\end{proof}

\subsection{Numerical Bounds}

To find better lower and upper bounds on $r_t$ we consider the following related problem.
The lower bound for instance can be improved by showing that for voter matrices with a high average majority $m_V$, i.e.\ a large fraction of Ys, a certain number of majority opinions can always be guaranteed.
Or we can pose the question the other way around.
How large does the average majority on the topics need to be such that the existence of a proposal with majority support that agrees with a certain number of majority opinions can be guaranteed?

\begin{definition}
For $t\in\mathbb{N}$ and $w\geq \ceil{\frac{t+1}{2}}$, let $ma_t(w)$ be the maximum average majority in a voter matrix $V\in \mathcal{V}_t$ with $md_V< w$.
\end{definition}

In other words, we are looking for the minimum number $ma_t(w)$, such that for all voter matrices with $m_V> ma_t(w)$, there exists a $w$-proposal that is supported by a majority.

Finding the values of $ma_t(w)$ is an interesting problem in itself. For instance, they are a natural extension of the Rule of Three-Forth in the sense that the rule can be derived from the value of $ma_t(t)$.
More precisely, the rule follows from the fact that $ma_t(t)<3/4$. This in turn follows from the following lemma.

\begin{lemma}\label{lem:ma_t(t)}
It holds
\begin{equation*}
    ma_t(t) = \frac{1}{2} + \frac{\floor{\frac{t-1}{2}}}{2t}.
\end{equation*}
\end{lemma}

\begin{proof}
\sloppy Any voter matrix containing more than $\floor{n/2}t + \ceil{n/2}\floor{(t-1)/2}$ Ys contains at least $\ceil{n/2}$ voters with at least $\floor{(t-1)/2}+1=\ceil{t/2}$ Ys, implying that the topic-wise majority proposal is supported by a majority. Hence,
\begin{equation*}
    ma_t(t)\leq \sup_{1\leq n\leq \infty}\left( \frac{\floor{\frac{n}{2}}t + \ceil{\frac{n}{2}}\floor{\frac{t-1}{2}}}{nt}\right)
    = \frac{1}{2} + \frac{\floor{\frac{t-1}{2}}}{2t}.
\end{equation*}
The matching lower bound on $ma_t(t)$ is a consequence of the following lemma.
\end{proof}

\begin{lemma}\label{lem:m_lower}
For $w\geq \ceil{(t+1)/2}$,
\begin{equation*}
    ma_t(w) \geq \frac{w}{2t} + \frac{\floor{\frac{t-1}{2}}}{2t}.
\end{equation*}
\end{lemma}

\begin{proof}
Consider the following voter matrix where the top $\ceil{n/2}$ voters are $(w-\ceil{(t+1)/2})$-voters.

\begin{figure}[h]
\centering
\begin{tikzpicture}[decoration={brace,amplitude=5pt}]
    \matrix (m) [matrix of nodes, column sep = 4.5pt, row sep = -2pt] {
    Y & $\cdots$ & Y & N & $\cdots$ & N\\
    $\vdots$ &  & $\vdots$ & $\vdots$ & & $\vdots$\\
    Y & $\cdots$ & Y & N & $\cdots$ & N\\
    Y & $\cdots$ & Y & Y & $\cdots$ & Y\\
    $\vdots$ &  & $\vdots$ & $\vdots$ & & $\vdots$\\
    Y & $\cdots$ & Y & Y & $\cdots$ & Y\\
    };
    \draw[decorate,thick] ([yshift=-2pt] m-1-6.north east) -- node[right=5pt] {$\ceil{\frac{n}{2}}$} ([yshift=2pt] m-3-6.south east);
    \draw[decorate,thick] ([yshift=-2pt] m-4-6.north east) -- node[right=5pt] {$\floor{\frac{n}{2}}$} ([yshift=2pt] m-6-6.south east);
\end{tikzpicture}
\end{figure}

None of the top $\ceil{n/2}$ voters support a proposal with $w$ or more Ys. This implies $md_V<w$. 
Furthermore, the average majority equals
\begin{equation*}
    \frac{\floor{\frac{n}{2}}t + \ceil{\frac{n}{2}}\left(w-\ceil{\frac{t+1}{2}}\right)}{nt} \geq \frac{\frac{n}{2}t + \frac{n}{2}\left(w-\ceil{\frac{t+1}{2}}\right)}{nt} = \frac{w}{2t} + \frac{\floor{\frac{t-1}{2}}}{2t}.
\end{equation*}
\end{proof}

\begin{figure}[ht]
    \centering
    \includegraphics[width=\columnwidth]{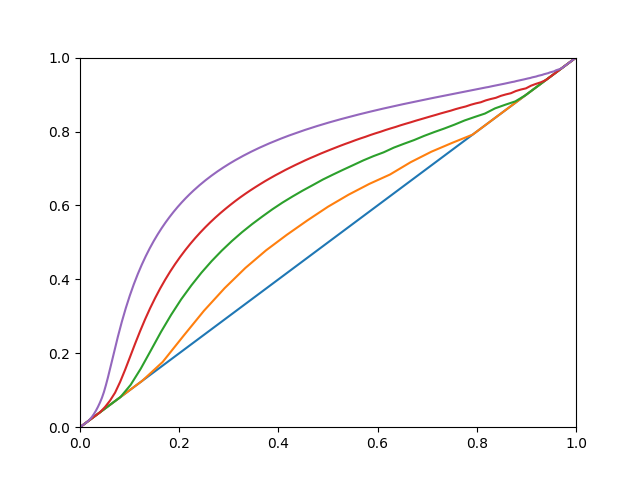}
    \caption{The normalized plot of the values $ma_t(w)$ from $w=\ceil{(t+1)/2}$ to $w=t$ for $t=9$ (blue), $t=49$ (yellow), $t=99$ (green), $t=199$ (red), and $t=499$ (purple). The x-axis plots $\frac{w-\ceil{(t+1)/2}}{\floor{(t-1)/2}}$, the y-axis plots $\frac{ma_t(w)-1/2}{\floor{(t-1)/2}/(2t)}$.}
    \label{fig:m_values}
\end{figure}

Let us examine the values of $ma_t(w)$ in more detail.
The lower bound from the previous lemma shows that $ma_t(w)$ increases at least linearly in $w$.
This bound is tight for $w=t$ (Lemma \ref{lem:ma_t(t)}) as well as for $w=\ceil{(t+1)/2}$.
The latter follows from Theorem \ref{thm:mdt_lower} as it guarantees the existence of a proposal with at least $\ceil{(t+1)/2}$ Ys with majority support for every matrix with an average majority of at least $1/2$.
So is the lower bound always tight?
By computing the values of $ma_t(w)$, we see that this is true for small $t$ (blue line in Figure \ref{fig:m_values}).
For larger $t$ however, it turns out that this is not the case and $ma_t(w)$ grows faster as Figure \ref{fig:m_values} shows.
Note that the family of voter matrices used for the upper bound in Theorem \ref{thm:r_t_general} all lie on the blue line when plotting average majority over $md_V$.
The fact that the other lines in Figure \ref{fig:m_values} lie above the blue line implies that such matrices will yield improved upper bounds on $r_t$.



\begin{figure*}[ht]
    \centering
    \includegraphics[width=0.7\textwidth]{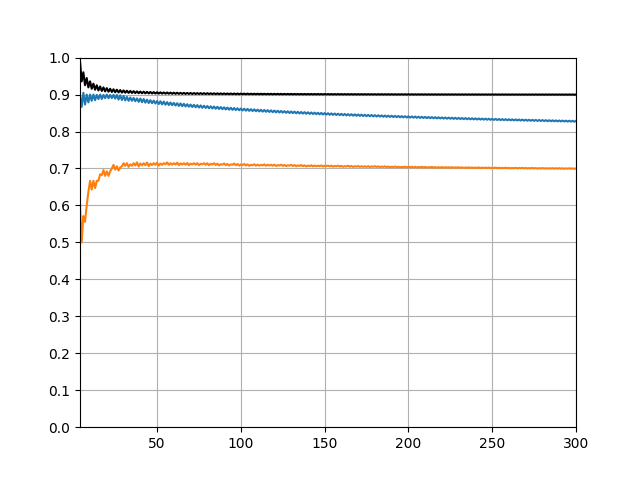}
    \caption{Lower (orange) and upper (blue) bounds on the achievable representativeness $r_t$ derived from the solutions of \eqref{eq:LP} for $t\in[4, 300]$. The black line shows the upper bound from Theorem \ref{thm:r_t_general} that converges to $~0.8989$ as $t \to \infty$.}
    \label{fig:representation_bounds}
\end{figure*}

In the following, we show that the values of $ma_t(w)$ can be computed by solving the following linear program, and study how they can be used to estimate $r_t$.

\begin{alignat}{3}
    &\text{maximize} & & \frac{1}{t}\sum_{l=0}^t lv_l' &\notag\\
    &\text{subject to} &\quad & \sum_{l=0}^t s_{k,l} v_l' < \frac{1}{2}s_{k,t}\tag{LP}\label{eq:LP}, & \quad k=w,\ldots,t\\
    & & & \sum_{l=0}^t v_l' = 1 &\notag
\end{alignat}

\begin{theorem}
For $w\geq \ceil{(t+1)/2}$, let $v_0^*,\ldots, v_t^*$ be a solution of \eqref{eq:LP}. Then $ma_t(w) = \sum_{l=0}^t lv_l^*$.
\end{theorem}

\begin{proof}
We first prove $ma_t(w)\leq \sum_{l=0}^t lv_l^*$.
Consider a voter matrix $V$ and for $l=0,\ldots,t$, denote by $v_l'$ the fraction of $l$-voters among all voters in $V$. That means $\sum_{l=0}^t v_l'=1$.
If $md_V<w$, then by Lemma \ref{lem:skl_inequality}
\begin{equation*}
    \sum_{l=0}^t s_{k,l} v_l' \leq \frac{1}{n}s_{k,t}\floor{\frac{n-1}{2}} < \frac{1}{2}s_{k,t}
\end{equation*}
for $k=w,\ldots,t$.
So the voter fractions of any $V$ with $md_V<w$ are a feasible solution of \eqref{eq:LP}.
Furthermore, the objective function of \eqref{eq:LP} describes the average majority in $V$.
Hence, $ma_t(w)\leq \sum_{l=0}^t lv_l'/t \leq \sum_{l=0}^t lv_l^*/t$.

It remains to argue that $ma_t(w)\geq \sum_{l=0}^t lv_l^*$.
From the solution $v_0^*,\ldots, v_t^*$, we construct a symmetric voter matrix $V_{LP}$ as follows.
For $l=0,\ldots,t$, we choose a fraction of about $v_l^*$ voters to be $l$-voters. By choosing $n$ sufficiently large, we can come arbitrarily close to the values $v_l^*$.
Then we multiply the number of voters by $\prod_{k=0}^t\binom{t}{k}$ such that for each $l=0,\ldots, t$, the number of $l$-voters is divisible by $\binom{t}{l}$.
Hence, we can for each $l=0,\ldots,t$ choose the specific $l$-voters in $V_{LP}$ such that all possible $\binom{t}{l}$ $l$-voters occur equally often.

The average majority in the resulting matrix is $\sum_{l=0}^t lv_l^* /t$. Furthermore, for $k=w,\ldots,t$ the total support of all $k$-proposals is
\begin{equation*}
    \sum_{l=0}^t s_{k,l} v_l \leq n \sum_{l=0}^t s_{k,l} v_l^* < \frac{n}{2}s_{k,t}=\frac{n}{2}\binom{t}{k}.
\end{equation*}
By our choice of the voters in $V_{LP}$, the matrix is symmetric in the topics in the following sense.
Permuting the topics does change the number of voters with a specific opinion vector. 
So all $\binom{t}{k}$ $k$-proposals have the same number of supporters, and this number must be smaller than $n/2$.
Therefore, no $k$-proposal with $k= w,\ldots,t$ is supported by $V_{LP}$, i.e.\ $md_{V_{LP}}<w$.
Finally, by construction of \eqref{eq:LP}, the average majority in $V_{LP}$ equals the value of the objective function of \eqref{eq:LP}.
This implies that $ma_t(w)\geq \sum_{l=0}^t lv_l^*$.
\end{proof}



The previously constructed matrix $V_{LP}$ can also be used to give an upper bound on $r_t$ using the values of $ma_t(w)$.
In the matrix, the average majority is $ma_t(w)$ and no proposal with $w$ or more Ys has a majority.
Hence,
\begin{equation}\label{eq:r_t_upper}
    r_t \leq \frac{(w-1)ma_t(w) + (t-w+1)(1-ma_t(w))}{t\cdot ma_t(w)}.
\end{equation}

Furthermore, the values of $ma_t(w)$ also yield a lower bound on $r_t$ by combining it with Lemma \ref{lem:1/2prop} in the following way.

\begin{table*}[h]
\centering
\begin{tabular}{ |c|c|c|c|c|c|c|c| } 
 \hline
 $t$ & 9 & 49 & 99 & 199 & 299 & 499 & 999\\ 
 \hline
 Lower bound & 0.6363 & 0.7098 & 0.7076 & 0.7028 & 0.6989 & 0.6943 & 0.6882\\ 
 \hline
 Upper bound & 0.8787 & 0.8756 & 0.8574 & 0.8379 & 0.8265 & 0.8124 & 0.7944\\
 \hline
\end{tabular}
\caption{Bounds on $r_t$ derived from the solutions of \eqref{eq:LP}. In each case, the 4 most significant figures are shown.}
\label{tbl:r_bounds}
\end{table*}

Consider an arbitrary matrix $V$ where the average majority on a topic is $m\in[0.5,1]$.
Let $w\in[\ceil{(t+1)/2},t]$ be the maximum number such that $ma_t(w)<m\leq ma_t(w+1)$ where we let $ma_t(t+1)=1$.
This implies that there exists a $w$-proposal with majority support.
Therefore,
\begin{align*}
    r_{V}
    \geq \frac{\frac{1}{2}w }{\frac{1}{2}w + 1\left(t-w\right)} = \frac{w}{2t-w}.
\end{align*}

On the other hand, we know from Lemma \ref{lem:1/2prop} that there exists a proposal $p$ with $R_p\geq 1/2-1/t$.
This implies
\begin{equation*}
    r_{V} \geq \frac{\frac{1}{2}-\frac{1}{t}}{m} \geq \frac{t-2}{2tma_t(w+1)}.
\end{equation*}

Therefore,
\begin{equation}\label{eq:r_t_lower}
    r_t \geq r_{V_{LP}} \geq \min_{w=\ceil{(t+1)/2},\ldots,t} \left( \max\left( \frac{w}{2t-w}, \frac{t-2}{2tma_t(w+1)} \right) \right).
\end{equation}


The lower and upper bounds resulting from the estimations \eqref{eq:r_t_lower} and \eqref{eq:r_t_upper} are plotted in Figure \ref{fig:representation_bounds} for all values of $t$ up to 300 together with the upper bound from Theorem \ref{thm:r_t_general}. Some selected values are also shown in Table \ref{tbl:r_bounds}.
All three bounds oscillate between even and odd $t$. They tend to be larger for even $t$ since then $t/2$ matching opinions already mean support compared to $(t+1)/2$ for odd $t$.


The remaining gap between the lower and upper bound stems from the fact that the upper bound examples we use have equal majorities for each topic but we did not prove that this is the worst case. If we knew that $\min_{V\in \mathcal{V}_t} r_V$ is obtained for a matrix with equal majorities, this would show the upper bound \eqref{eq:r_t_upper} to be tight.

\section{Conclusion}
In this paper, we have quantified the price of majority support.
This gives us a better understanding of how much representativeness needs to be sacrificed if majority support is required.
We suspect that the upper bound derived from $ma_t(w)$ (blue line in Figure \ref{fig:representation_bounds}) is tight. It remains as future work to prove this and to find the limit it converges to for $t\to \infty$.

Furthermore, this paper opens several possibilities for future work:
First, this work provides guarantees on the representativeness of majority supported proposals.
One could study the frequency of the examples that limit these guarantees (upper bound examples) by either assuming a distribution on the voter opinions or using real world data.
Another possibility is to extend the problem to more than two parties: The results of this paper quantify how well voters can be represented in a system of two (opposed) parties. How well can they be represented by a larger number of parties?
Finally, since we focus on binary issues in this paper, a further possible extension is to study multiple choice or continuous decisions, or issues which are not mutually independent (as in judgement aggregation).

\appendix
\section{Appendix}

\begin{lemma}\label{lem:bin_sum_app}
For $l=0,\ldots, t$,
\begin{equation*}
    \sum_{k = \ceil{t/2}}^{t} (2k-t) s_{k,l} = l \binom{t-1}{\floor{t/2}}.
\end{equation*}
\end{lemma}
\begin{proof}
Let
\begin{equation*}
    f(l) = \sum_{k = \ceil{t/2}}^{t} (2k-t) s_{k,l}.
\end{equation*}
Note that we use the convention that $\binom{n}{k}=0$ for $k>n$ and $k<0$. Hence, the upper summation bound in the formula for $s_{k,l}$ from Lemma \ref{lem:s_kl} can be omitted. 
Inserting this formula yields
\begin{align*}
    f(l) &= \sum_{k=\ceil{t/2}}^{t} \sum_{x=\ceil{(k+l-\floor{t/2})/2}}^{\infty} \binom{l}{x}\binom{t-l}{k-x} (2k-t)\\
    &= \sum_{x=\ceil{(l+1)/2}}^\infty \binom{l}{x} \sum_{k=\ceil{t/2}}^{2x-l+\floor{t/2}} \binom{t-l}{k-x} (2k-t)\\
    &= \sum_{x=\ceil{(l+1)/2}}^\infty \binom{l}{x} \sum_{y=\ceil{t/2}-x}^{t-l-(\ceil{t/2}-x)} \binom{t-l}{y} (2y+2x-t).
\end{align*}
We swapped summations in the second step and substituted $y=k-x$ in the third step.
Note that
\begin{equation*}
    \binom{t-l}{y}(2y+2x-t)+\binom{t-l}{t-l-y}(2(t-l-y)+2x-t) = 2\binom{t-l}{y}(2x-l).
\end{equation*}
Using this we further conclude
\begin{align*}
    f(l) &= \sum_{x=\ceil{(l+1)/2}}^\infty \binom{l}{x} \sum_{y=\ceil{t/2}-x}^{x-l+\floor{t/2}} \binom{t-l}{y} (2x-l)\\
    &= \sum_{y=\ceil{t/2}-l}^{\floor{t/2}}\binom{t-l}{y} \sum_{x=\max(\ceil{t/2}-y,y+l-\floor{t/2})}^\infty \binom{l}{x} (2x-l).
\end{align*}
In the second step, we switched the summation again.
Now let $x_0=\max(\ceil{t/2}-y,y+l-\floor{t/2})$. Then
\begin{align*}
    \sum_{x=x_0}^\infty \binom{l}{x} (2x-l)
    &= \sum_{x=x_0}^\infty x\binom{l}{x} - (l-x)\binom{l}{x}\\
    &= \sum_{x=x_0}^\infty l\binom{l-1}{x-1} - l\binom{l-1}{x}
    = l\binom{l-1}{x_0-1}.
\end{align*}
Furthermore, the definition of $x_0$ implies
\begin{equation*}
    \binom{l-1}{\floor{t/2}-y} = \binom{l-1}{y+l-\ceil{t/2}} = \binom{l-1}{x_0-1}.
\end{equation*}
With the previous two properties, we establish
\begin{align*}
    f(l) &= \sum_{y=\ceil{t/2}-l}^{\floor{t/2}}\binom{t-l}{y} l \binom{l-1}{\floor{t/2}-y}\\
    &= l \sum_{z=0}^{l-1} \binom{t-l}{\floor{t/2}-z}\binom{l-1}{z} = l\binom{t-1}{\floor{t/2}}.
\end{align*}
Here we substituted $z=\floor{t/2}-y$, and the last step follows from the well-known combinatorial identity $\binom{n}{k} = \sum_{j} \binom{i}{j}\binom{n-i}{k-j}$.
\end{proof}

\begin{acks}
We thank the anonymous reviewers for their helpful comments and suggestions.
\end{acks}



\bibliographystyle{ACM-Reference-Format} 
\bibliography{aamas22.bib}


\end{document}